\documentclass[11pt]{article}
\usepackage{cite}
\usepackage{amsmath,amssymb,amsfonts}
\usepackage{algorithmic}
\usepackage{graphicx}
\usepackage{textcomp}
\usepackage{xcolor}

\usepackage{algorithm}
\usepackage{amsthm}
\newtheorem{theo}{Theorem}
\newtheorem*{proof*}{Proof}
\newtheorem{lemma}{Lemma}
\newtheorem{defi}{Definition}

\begin{document}

\title{Network Implosion: Effective Model Compression for ResNets via Static Layer Pruning and Retraining}

\author{Yasutoshi Ida \\ NTT Software Innovation Center \\ 3-9-11, Midori-cho Musashino-shi, Tokyo, Japan \\ ystsh521@gmail.com \and 
Yasuhiro Fujiwara \\ NTT Software Innovation Center \\ 3-9-11, Midori-cho Musashino-shi, Tokyo, Japan \\ fujiwara.yasuhiro@lab.ntt.co.jp}

\maketitle

\begin{abstract}
Residual Networks with convolutional layers are widely used in the field of machine learning.
Since they effectively extract features from input data by stacking multiple layers, they can achieve high accuracy in many applications.
However, the stacking of many layers raises their computation costs.
To address this problem, we propose Network Implosion, it erases multiple layers from Residual Networks without degrading accuracy.
Our key idea is to introduce a priority term that identifies the importance of a layer; we can select unimportant layers according to the priority and erase them after the training.
In addition, we retrain the networks to avoid critical drops in accuracy after layer erasure.
A theoretical assessment reveals that our ``erasure and retraining'' scheme can erase layers without accuracy drop, and achieve higher accuracy than is possible with training from scratch.
Our experiments show that Network Implosion can, for classification on Cifar-10/100 and ImageNet, reduce the number of layers by 24.00\%\(\sim\)42.86\% without any drop in accuracy.
\end{abstract}

\section{Introduction}
Convolutional Neural Networks (CNNs) \cite{lenet5} are important tools in the machine learning community because they have a wide field of applications such as
scene labeling \cite{icml1},
online tracking \cite{icml2}, graph classification \cite{icml3}, sequence to sequence learning \cite{icml4}, 
language modeling \cite{icml5}, and protein structure prediction \cite{icml6}.
Although modern CNNs need a lot of time for the training phase, it can be shortened by using many computation resources.
In fact, \cite{akiba} shows that large CNNs can be trained within 15 minutes on ImageNet datasets by using 1024 GPUs.
On the other hand, after the training phase, the inference phase is used to perform prediction in service deployment.
Since the era of IoT has arrived, it is increasingly important to perform inference on devices with limited resources such as image classification on embedded systems \cite{qcnn}, character recognition on portable devices \cite{cr} and speech recognition on mobile devices \cite{mobilesr}.

While we need to perform inference with limited resources, the number of layers in CNNs has been increasing every year in order to raise accuracy.
In 1998, LeNet-5 used 5 layers to classify handwriting digits \cite{lenet5}.
In 2012, AlexNet won an ILSVRC image classification competition with 8 layers \cite{alexnet}.
In the competition of 2014, VGG Net and GoogleNet stacked 19 and 22 layers, respectively \cite{vgg,googlenet}.
Residual Network (ResNet) used 152 layers and won the competition in 2015 \cite{resnet}.
The paper of ResNet has many citations, more than 5,000 just within the last two years, and ResNet is being used as a standard CNN-based model.

However, due to its sheer number of layers, ResNet incurs considerable computation overheads such as processing time and memory usage.
Although we can efficiently train CNNs by using GPUs, it is still difficult to perform inference efficiently with limited resources such as embedded systems and mobile devices.
Several approaches reduce the number of layers in performing the inference phase \cite{fractalnet,DDN,wrn,gate_resnet,blockdrop,sss}.
Unfortunately, they incur additional memory requirements or degrade the accuracy on real-world datasets such as ImageNet.
Therefore, we need other strategies that reduce the number of layers without increasing memory consumption while keeping accuracy high.

To achieve this goal, we propose Network Implosion (NI); it erases multiple layers from ResNet without increasing the computation costs in the inference phase.
Our proposal introduces a priority term that indicates the importance of each layer.
We can select and erase unimportant layers according to the priority after training the network.
In addition, our method can avoid any critical drop in accuracy by retraining the network after erasure.
We analyze this ``erasure and retraining'' scheme by using the theories underlying hyperplane arrangements \cite{pascanu2013number} and generalization error bound \cite{koltchinskii}.
Our analysis reveals that our ``erasure and retraining'' scheme can erase layers without accuracy drop, and achieve higher accuracy than is possible with the usual training.
This is because we can perform retraining with preferable initial parameters in terms of the generalization error bound.
Our experiments show that NI reduces the number of layers in ResNet with no additional computation costs in the inference; for classification tasks on Cifar-10/100 and ImageNet, the layer reductions are 57.14\%\(\sim\)76.00\%.

\section{Related work}
{\bf Dynamic Layer Pruning.}
\cite{ensambles} found that ResNet does not suffer a significant loss of accuracy if a few layers are erased.
In their experiments, when layers were erased from networks, ResNet suffered only a slight drop in accuracy while the accuracy of the VGG architecture \cite{vgg} dropped significantly in the inference phase.
Inspired by the results gained, some recent papers erase layers from ResNet in order to raise processing speed.
\cite{blockdrop,gate_resnet} dynamically erase layers that are not needed during the inference phase.
These ``Dynamic Layer Pruning'' can easily keep the accuracy, however, they can not reduce the memory consumption.
Far from it, they have additional modules that need additional memory.
\cite{blockdrop} needs an additional network to determine which layers can be skipped.
\cite{gate_resnet} also needs additional gating functions that decide which layers to be skipped.

{\bf Static Layer Pruning.}
``Static Layer Pruning'' completely erases layers while Dynamic Layer Pruning only selects layers to be removed during the inference.
Thus Static Layer Pruning is preferable in terms of the computation cost for the inference because it reduces the memory consumption while Dynamic Layer Pruning cannot.
\cite{stsp,sss} utilize sparse regularizations \cite{sp1,sp2} that erase layers from ResNet.
\cite{epsilon} ignores layers that have subthreshold activations.
However, since these methods need to tune the continuous hyper parameters of the regularization or the threshold, it is difficult to obtain the desired number of layers.
In the industrial usage, it is preferable to directly determine the number of layers because of the hardware requirements.

{\bf Teacher-Student Training.}
The teacher-student training regime is a well-known algorithm that trains shallow student networks by using deep trained teacher networks \cite{DDN,KD}.
In teacher-student training, the shallow student networks can be effectively trained because the deep trained teacher network gives the probability distribution over the classes to the student networks in order to boost their training.
Since we can freely design the student network, we can reduce the number of layers without additional computation costs in the inference.
In addition, \cite{KD} reports that teacher-student training improves the accuracy on several datasets and tasks.

Other several papers also try to reduce the number of layers in deep neural networks.
FractalNet \cite{fractalnet} can reduce the number of layers for the inference phase without increasing the parameters; unfortunately, it degrades the accuracy when a 20-layer model was used instead of a 40-layer model.
\cite{wrn} showed how to achieve high accuracy by increasing the number of parameters in each layer even if the network is shallow.
However, their 50-layer model has more parameters than the usual 200-layer model.

\section{Preliminary}
This section introduces ResNet, which is now widely used as a standard CNN-based model.
ResNets have blocks called Residual Units \cite{resnet}.
Since each Residual Unit has multiple convolutional layers, ResNets can have deep architectures by stacking Residual Units.
The $l$-th Residual Unit introduced in \cite{1001layer} is defined as follows:
\begin{align}
\label{residual_unit}
{\bf x}_{l+1} = {\bf x}_{l} + F({\bf x}_{l}),
\end{align}
where \({\bf x}_{l}\) is the input to the \(l\)-th Residual Unit.
\(F(\cdot)\) is a module that consists of convolutional layers, batch normalizations \cite{bn} and Rectified Linear Units (ReLUs) \cite{alexnet}.
Therefore, each Residual Unit performs identity mapping of \({\bf x}_{l}\) and nonlinear mapping of \(F(\cdot)\).
ResNets also have the structure called {\it stages} \cite{unroll}.
A stage has several stacked Residual Units of the same dimensionality for inputs and outputs.
Note that the dimensionality can be changed between stages by using down-sampling and increasing channels of convolutions.

In terms of layer pruning for ResNet, Static Layer Pruning is preferable as compared with Dynamic Layer Pruning because it reduces all computation cost.
However, previous approaches cannot directly decide the number of layers.
In addition, it is difficult for Static Layer Pruning to recover the accuracy because it completely erases layers.
Next section introduces our method of Static Layer Pruning, Network Implosion that can directly decide the number of layers, and keep the accuracy with theoretical assessments.

\section{Proposed method}
\label{network_implosion}
We propose Network Implosion, which is allocated to Static Layer Pruning.
Our method can effectively recover accuracy by utilizing ``erasure and retraining'' scheme.
In this section, we theoretically show the reason why erasure and retraining scheme is effective for Static Layer Pruning.
In particular, we show that the trade off between representational power and generalization bound holds when we use erasure and retraining scheme.
We first describe that the representational power of ResNet can be reduced after we erase a layer in terms of the theory of hyperplane arrangements \cite{pascanu2013number}.
Next, we show that the procedure of erasure and retraining can effectively recover the accuracy in terms of the generalization error bound theory \cite{koltchinskii}.
Then, we introduce the algorithm of Network Implosion which effectively erase multiple layers by employing erasure and retraining scheme.

\subsection{Erasure and representational power}
In this section, we show that the representational power is retained but can be reduced when we erase layers form ResNet.
\cite{pascanu2013number} shows that neural networks with ReLUs divide the input space into several linear regions.
Since many linear regions can approximate a complex curved boundary, we define the representational power as the number of linear regions as in \cite{pascanu2013number}.
In order to make descriptions simple, we use following model for the analysis:
\begin{defi}[L-layered FC-ResNet]
Consider \(L\)-layered ResNet such that the Residual Unit \(F(\cdot)\) consists of a fully-connected layer and ReLU activations.
Let \({\bf x} \in [-N,N]^{n_{0}}\) be the input, and \(\hat{W}_{l}\) is the upper bound of the summation of parameters in absolute values for the \(l\)-th layer.
If \(n_{i} \geq n_{0}\) for \(i \in [L] := \{1,...,L\}\) where \(n_{i}\) is the number of outputs at \(i\)-th layer, we say the model is an \(L\)-layered FC-ResNet.
\end{defi}
Although we use the above \(L\)-layered FC-ResNet throughout the analysis, we can easily expand the discussion to the more general case as shown later.
In the case of \(L\)-layered FC-ResNet, we derive the maximal lower bound for the maximal number of linear regions in the following lemma:
\begin{lemma}
\label{lemma1}
If \(\phi\) is the maximal number of linear regions in the input space, the function of \(L\)-layered FC-ResNet has
\(\textstyle{\left(\prod_{i=1}^{L-1}\left\lfloor\frac{n_{i}}{n_{0}}\right\rfloor^{n_{0}}\right)\sum_{j=0}^{n_{0}}\binom{n_{L}}{j}}\)
as the maximal lower bound for \(\phi\).
\end{lemma}
\begin{proof}
Since ResNet has \(L\) Residual Units, an input signal can pass through \(2^{L}\) paths.
This is because a Residual Unit has two paths that are represented as \({\bf x}_{l+1} = {\bf x}_{l} + F({\bf x}_{l})\).
Thus, there exists a path that has \(L\) nonlinear mappings of \(F(\cdot)\).
This path can be seen as an \(L\)-layered fully-connected rectifier network, and its lower bound in the maximal number of linear regions is \(\textstyle{\left( \prod_{i=1}^{L-1} \left\lfloor \frac{n_{i}}{n_{0}} \right\rfloor^{n_{0}} \right) \sum_{j=0}^{n_{0}} \binom{n_{L}}{j}}\) as described in \cite{NIPS2014_5422}.
Although other paths also have lower bounds, they are obviously smaller than \(\textstyle{\left( \prod_{i=1}^{L-1} \left\lfloor \frac{n_{i}}{n_{0}} \right\rfloor^{n_{0}} \right) \sum_{j=0}^{n_{0}} \binom{n_{L}}{j}}\).
Therefore, we have  \(\textstyle{\left( \prod_{i=1}^{L-1} \left\lfloor \frac{n_{i}}{n_{0}} \right\rfloor^{n_{0}} \right) \sum_{j=0}^{n_{0}} \binom{n_{L}}{j}}\) in Lemma \ref{lemma1} as a maximal lower bound.
\end{proof}
Lemma \ref{lemma1} yields the representational power of ResNet.
In terms of erasing layers, we prove that ResNet has its representational power even if we erase layers.
The following lemma yields the maximal lower bound for the maximal number of linear regions after we erase a layer from ResNet:
\begin{lemma}
\label{lemma2}
If \(\phi\) is the maximal number of linear regions in the input space, when we erase the \(l'\)-th Residual Unit from \(L\)-layered FC-ResNet where \(1\leq l'<L\), the function of ResNet has \(\textstyle{\left( \prod_{i \in \{[L-1] \setminus l'\} } \left\lfloor \frac{n_{i}}{n_{0}} \right\rfloor^{n_{0}} \right) \sum_{j=0}^{n_{0}} \binom{n_{L}}{j}}\) as the maximal lower bound of \(\phi\).
\end{lemma}
\begin{proof}
If we erase the \(l'\)-th Residual Unit, we can remove the element of \(l'\) from \([L]\).
Therefore, we have \(\textstyle{\left( \prod_{i \in \{[L-1] \setminus l'\} } \left\lfloor \frac{n_{i}}{n_{0}} \right\rfloor^{n_{0}} \right) \sum_{j=0}^{n_{0}} \binom{n_{L}}{j}}\) as the bound in Lemma \ref{lemma2} from the bound in Lemma \ref{lemma1}.
\end{proof}
Lemma \ref{lemma2} suggests that ResNet retains its representational power even if we erase a Residual Unit.
We can also obviously expand the lemma to the case where we erase several Residual Units.
Thus, we can erase multiple Residual Units from ResNet on the base of Lemma \ref{lemma2}.
However, Lemma \ref{lemma2} also suggests that the erasure may degrade the representational power, and the accuracy.
Therefore, we retrain ResNet in order to recover the accuracy after we erase layers.
Next section shows that retraining is theoretically effective for recovering the accuracy in terms of generalization error bound.

\subsection{Retraining and generalization}
\label{theory}
Although the accuracy may degrade if we erase layers as described in the previous section,
we show that the accuracy can be effectively recovered by retraining the model after the erasure in this section.
In particular, we can attain a tighter bound of generalization error with retraining after erasing layers than is possible with usual training.
Intuitively, our method often achieves higher accuracy than the usual training regime since we can retrain the model with preferable initial parameters in terms of generalization error bound.

In order to show the bound of generalization error, we introduce some definitions of classifier, margin, generalization error, empirical error with margin and Rademacher average.
\begin{defi}[Classifier]
Let \(\mathcal{Y}\) be a set of labels and \(M\) be its cardinality.
 \(\mathcal{X}\in\mathcal{R}^{n_{0}}\) is the input space for a model \(f\in\mathcal{F}:\mathcal{X}\times\mathcal{Y}\to\mathcal{R}\).
If the model \(f(x,y)\) predicts a probability of label \(y\in\mathcal{Y}\) for an input \(x\in\mathcal{X}\), we say the model \(f\) is a classifier.
\end{defi}
\begin{defi}[Margin]
Let \(x\in\mathcal{X}\) and \(y\in\mathcal{Y}\) be an input and a label, respectively.
If \(f\) is a classifier of the multi-label classification problem, the margin \(\gamma_{f}(x,y)\) is defined as \(f(x,y)-{\rm max}_{z \neq y}f(x,z)\).
\end{defi}
\begin{defi}[Generalization error]
Let \({\rm \bar{E}}[\cdot]\) be the expectation with respect to the joint distribution over \(\mathcal{X}\times\mathcal{Y}\) and \(\mathcal{I}_{\{\cdot\}}\) be the indicator function.
If we have a classifier \(f\) of the multi-label classification problem, the generalization error \(\mathcal{E}_{g}(f)\) is \(\textstyle{{\rm \bar{E}}[\mathcal{I}_{\{{\rm arg \ max}_{z\in\mathcal{Y}}f(x,z)\neq y\}}]={\rm \bar{E}}[\mathcal{I}_{\{\gamma_{f}(x,y)<0\}}]}\).
\end{defi}
\begin{defi}[Empirical error with margin]
Let \(d\) be the quantity of training data.
If \(f\) is a classifier of the multi-label classification problem, the empirical error with margin \(\mathcal{E}_{e}^{\rho}(f)\) is \(\textstyle{\frac{1}{d}\sum^{d}_{i=1}\mathcal{I}_{\{\gamma_{f}(x_{i},y_{i})\leq\rho\}}}\)
where \(\rho\) is called the {\it margin coefficient}; it takes the value of \(\rho > 0\).
\end{defi}
\begin{defi}[Rademacher average]
For class \(\mathcal{F}\) of functions, Rademacher average \(R_{m}(\mathcal{F})\) is defined as \(\textstyle{{\rm E}_{x,\sigma}\Bigl[\underset{f\in\mathcal{F}}{\rm sup} |\frac{2}{m}\sum_{i=1}^{m}\sigma_{i}f(x_{i})|\Bigr]}\) where \(x=\{x_{1},...,x_{m}\}\) that are samples from distribution over \(\mathcal{X}\), and \(\{\sigma_{1},...,\sigma_{m}\}\) are i.i.d. samples with \(P(\sigma_{i}=1)=1/2\) and \(P(\sigma_{i}=-1)=1/2\).
\end{defi}
By using the above definitions, we introduce a known bound of generalization error by following \cite{koltchinskii}:
\begin{lemma}
\label{lemma3}
Suppose \(f\in F\) is a classifier of the multi-label classification problem.
For \(\forall\delta>0\) and \(\forall f \in \mathcal{F}\), we have the following bound of generalization error with probability \(1-\delta\):
%\vskip -2em
\begin{align}
\label{lemma3eqn}
\textstyle{\mathcal{E}_{g}(f) \leq \mathcal{E}_{e}^{\rho}(f) + \frac{8M(2M - 1)}{\rho} R_{m}(\mathcal{\hat{F}}) + \sqrt{\frac{{\rm log log}_{2}\frac{2}{\rho}}{m}}+\sqrt{\frac{{\rm log}\frac{2}{\delta}}{2m}}, }
\end{align}
where \(\mathcal{\hat{F}}=\{x \to f(\cdot,k);k\in\mathcal{Y},f\in\mathcal{F}\}\).
\end{lemma}
Lemma \ref{lemma3} suggests that if the bound of (\ref{lemma3eqn}) is loose, \(\mathcal{E}_g(f)\) can be large and the accuracy may degrade in the test phase.
Notice that the generalization error \(\mathcal{E}_g(f)\) is bounded by Rademacher average \(R_{m}(\mathcal{\hat{F}})\) and the empirical error \(\mathcal{E}_{e}^{\rho}(f)\).
In other words, if we can reduce the values of \(R_{m}(\mathcal{\hat{F}})\) and \(\mathcal{E}_{e}^{\rho}(f)\), we can make the bound of \(\mathcal{E}_g(f)\) tight.
Therefore, we explain how the bound of \(\mathcal{E}_g(f)\) changes in terms of \(R_{m}(\mathcal{\hat{F}})\) and \(\mathcal{E}_{e}^{\rho}(f)\) as we erase layers.

In terms of Rademacher average \(R_{m}(\mathcal{\hat{F}})\), we show that the upper bound of \(R_{m}(\mathcal{\hat{F}})\) reduces by erasing layers.
The upper bound of \(R_{m}(\mathcal{\hat{F}})\) of \(L\)-layered FC-ResNet is represented as follows:
\begin{lemma}
\label{lemma4}
For function class \(\mathcal{\hat{F}}_{L}\) of \(L\)-layered FC-ResNet classifiers for the multi-label classification problem,
we have a maximal upper bound of \(R_{m}(\mathcal{\hat{F}}_{L})\) as follows:
\begin{align}
\label{lemma4eqn}
\textstyle{R_{m}(\mathcal{\hat{F}}_{L})\leq cN\sqrt{\frac{{\rm log} \ n_{0}}{m}}\prod_{l=1}^{L}\hat{W}_{l}},
\end{align}
where \(c\) is a constant.
\end{lemma}
\begin{proof}
From Lemma \ref{lemma1}, \(L\)-layered FC-ResNet has a path that consists of \(L\) nonlinear mappings of \(F(\cdot)\) and can be seen as an \(L\)-layered fully connected rectifier network.
Therefore, given \(x\in\mathcal{X}\) and \(k\in\mathcal{Y}\), we decompose \(\mathcal{\hat{F}_{L}}\) into \(L\)-layered fully connected rectifier networks by using parameter \(w_{i}\in\mathcal{R}\) and activation \(f_{i}\) from the previous layer.
We represent this as follows:
\begin{align}
\label{f1}
\textstyle{\mathcal{\hat{F}}_{L} = \{ (x,k) \to \sum_{i=1}^{n_{L-1}} w_{i}f_{i}(x);f_{i}\in\mathcal{\bar{F}}_{L-1}\}}.
\end{align}
Here, \(\mathcal{\bar{F}}_{l}\) for \(l=1,...,L-1\) is represented as follows:
\begin{align}
\label{f2}
\textstyle{\mathcal{\bar{F}}_{l}=\{ x \to \phi(f(x)); f\in\mathcal{\hat{F}}_{l}\}},
\end{align}
where \(\phi(\cdot)\) is the ReLU activation function defined as \({\rm max}(0,\cdot)\).
In the equation, \(\mathcal{\hat{F}}_{l}\) for \(l=1,...,L-1\) is represented as follows:
\begin{align}
\label{f3}
\textstyle{\mathcal{\hat{F}}_{l} = \{ x \to \sum_{i=1}^{n_{l-1}} w_{i}f_{i}(x);f_{i}\in\mathcal{\bar{F}}_{l-1}\}}.
\end{align}
Note that \(\mathcal{\bar{F}}_{0}=\{x\to x_{i};i \in n_{0}\}\).
Given (\ref{f1}), (\ref{f2}) and (\ref{f3}), we have \(\textstyle{R_{m}(\mathcal{\hat{F}}_{L})\leq R_{m}(\mathcal{\hat{F}}_{1}) \prod_{l=2}^{L}\hat{W}_{l}}\) according to \cite{sun}.
In addition, we have \(\textstyle{R_{m}(\mathcal{\hat{F}}_{1})}\)
\(\textstyle{\leq c\hat{W}_{1}N\sqrt{\frac{{\rm log} \ n_{0}}{m}}}\) by following \cite{berlett}.
By combining these two inequalities, we obtain (\ref{lemma4eqn}).
Although there exist other paths in \(L\)-layered FC-ResNet, the number of nonlinear mappings is obviously smaller than \(L\).
Therefore, (\ref{lemma4eqn}) is a maximal upper bound.
\end{proof}
Lemma \ref{lemma4} suggests that if \(\hat{W}_{l}>1\), the bound of Rademacher average can be exponentially large when we use a large number of layers.
Thus, the bound of the generalization error in Lemma \ref{lemma3} can also be large.
The similar observations are also suggested in \cite{Neyshabur}.
They show that deep architectures can suffer a loose bound on generalization error because it depends on exponent of the number of layers similar to (\ref{lemma4eqn}).
Therefore, the large number of layers can negatively impact the accuracy in the test phase.
Next, we obtain the bound after erasing a layer by using Lemma \ref{lemma2} and \ref{lemma4} as follows:
\begin{lemma}
\label{lemma5}
Let \(\mathcal{\hat{F}}_{L}\) be function class of \(L\)-layered FC-ResNet classifiers for multi-label classification problem.
Suppose that we have function class \(\mathcal{\hat{F}}_{[L]\setminus l'}\) as a result of the erasure of the \(l'\)-th Residual Unit from ResNet where \(1 \leq l' < L\).
Then, we have the following maximal upper bound of \(R_{m}(\mathcal{\hat{F}}_{[L]\setminus l'})\):
\begin{align}
\label{lemma5eqn}
\textstyle{R_{m}(\mathcal{\hat{F}}_{[L]\setminus l'})\leq cN\sqrt{\frac{{\rm log} \ n_{0}}{m}}\prod_{l\in\{[L]\setminus l'\}}\hat{W}_{l}},
\end{align}
where \(c\) is a constant.
\end{lemma}
\begin{proof}
From Lemma \ref{lemma2}, we can remove the element of \(l'\) from \([L]\).
Thus, we have (\ref{lemma5eqn}) in Lemma \ref{lemma5} from (\ref{lemma4eqn}) in Lemma \ref{lemma4}.
\end{proof}
Lemma \ref{lemma5} suggests that \(\hat{W}_{l'}\) can affect the tightness of the bound.
In addition, if we erase a layer from \(L\)-layered FC-ResNet, we can derive the following lemma from Lemma \ref{lemma4} and \ref{lemma5}:
\begin{lemma}
\label{lemma6}
Suppose that \(\mathcal{\hat{F}}_{L}\) is function class of \(L\)-layered FC-ResNet classifiers such that \(\hat{W}_{l'}>1\) for the multi-label classification problem, and \(\mathcal{\hat{F}}_{[L]\setminus l'}\) is function class of classifiers that erases the \(l'\)-th Residual Unit from \(\mathcal{\hat{F}}_{L}\).
Let \(\bar{R}_{m}(\cdot)\) be an upper bound for \(R_{m}(\cdot)\).
Then, we have the following inequality for \(\mathcal{\hat{F}}_{[L]\setminus l'}\) and \(\mathcal{\hat{F}}_{L}\):
\begin{align}
\label{lemma6eqn}
\textstyle{\bar{R}_{m}(\mathcal{\hat{F}}_{[L]\setminus l'})< \bar{R}_{m}(\mathcal{\hat{F}}_{L})}.
\end{align}
\end{lemma}
\begin{proof}
Since the values of \(\hat{W}_{l}\) is the same in \(\mathcal{\hat{F}}_{L}\) and \(\mathcal{\hat{F}}_{[L]\setminus l'}\), the values of \(\hat{W}_{l}\) for (\ref{lemma4eqn}) in Lemma \ref{lemma4} are the same as those of (\ref{lemma5eqn}) in Lemma \ref{lemma5}.
Thus, we have (\ref{lemma6eqn}) in Lemma \ref{lemma6} if the condition of \(\hat{W}_{l'}>1\) holds.
\end{proof}
Lemma \ref{lemma6} suggests that the bound of Rademacher average can be tight by erasing a layer.
Therefore, we often retrain the models having the tight bound of Rademacher average by reusing parameters after training and erasure.
Notice that if we use function class \(\mathcal{\hat{F}}_{L-1}\) of \((L-1)\)-layered FC-ResNet classifiers instead of \(\mathcal{\hat{F}}_{[L]\setminus l'}\) in Lemma \ref{lemma6}, (\ref{lemma6eqn}) may not hold.
This is because the values of \(\hat{W}_{l}\) in \(\mathcal{\hat{F}}_{L-1}\) may not be the same as those of \(\mathcal{\hat{F}}_{L}\).
As a result, it is possible to have \(\textstyle{\bar{R}_{m}(\mathcal{\hat{F}}_{L-1})\geq \bar{R}_{m}(\mathcal{\hat{F}}_{L})}\) instead of (\ref{lemma6eqn}) since the values of \(\hat{W}_{l}\) are not common in (\ref{lemma4eqn}) and (\ref{lemma5eqn}).

In terms of empirical error \(\mathcal{E}_{e}^{\rho}(f)\) where \(f\in\mathcal{\hat{F}}_{L}\), ResNet can substantially reduce the value of \(\mathcal{E}_{e}^{\rho}(f)\) because the representational power exponentially grows in Lemma \ref{lemma1} as the number of layers is increased.
In addition, as described in the section of related work, ResNet does not suffer any significant drop in accuracy if we erase a few layers after training \cite{ensambles}.
This means that, for trained classifier \(f\), we have the condition of \(\textstyle{\mathcal{E}^{\rho}_{e}(f')-\mathcal{E}^{\rho}_{e}(f)}=\epsilon\) where \(f'\in\mathcal{\hat{F}}_{[L]\setminus l'}\) and \(\epsilon\) is a small positive value.
Therefore, we can expect the tight bound of \(\mathcal{E}_g(f)\) even if we erase a few layers after training.

Finally, we have following theorem for the generalization error bounds of \(f\in\mathcal{\hat{F}}_{L}\) and \(f'\in\mathcal{\hat{F}}_{[L]\setminus l'}\):
\begin{theo}
\label{theorem1}
Let \(\mathcal{\bar{E}}_g(\cdot)\) be an upper bound of generalization error, and \(\rho\) be a fixed margin.
Suppose that Lemma \ref{lemma6} holds, and \(\mathcal{E}_{e}^{\rho}(f)<\mathcal{E}_{e}^{\rho}(f')\) where \(f\in\mathcal{\hat{F}}_{L}\) is a trained \(L\)-layered FC-ResNet classifier such that \(\hat{W}_{l'}>1\) for the muti-label classification problem and \(f'\in\mathcal{\hat{F}}_{[L]\setminus l'}\) is a classifier that erases the \(l'\)-th Residual Unit from \(f\).
For \(\forall\delta>0\) and \(\forall f\in\mathcal{F}\), when the condition \(\textstyle{\mathcal{E}^{\rho}_{e}(f')-\mathcal{E}^{\rho}_{e}(f)<\frac{8M(2M-1)}{\rho}(\bar{R}_{m}(\mathcal{\hat{F}}_{L})-\bar{R}_{m}(\mathcal{\hat{F}}_{[L]\setminus l'}))}\) holds,
we have the bound of \(\mathcal{\bar{E}}_{g}(f')<\mathcal{\bar{E}}_{g}(f)\) with probability \((1-\delta)^2\).
\end{theo}
\begin{proof}
From Lemma \ref{lemma3} and the upper bound of Rademacher average, we first have the following equation:
\(\textstyle{\mathcal{\bar{E}}_{g}(\cdot)}=
\mathcal{E}_{e}^{\rho}(\cdot)+8M(2M - 1)\rho^{-1} \bar{R}_{m}(\mathcal{\cdot}) + \sqrt{{m^{-1}\rm log log}_{2}(2\rho^{-1})}+\sqrt{(2m)^{-1}{\rm log}(2\delta^{-1})}\) with fixed \(\rho\).
Next, from the above equation, we have the following equation:
\begin{align}
\label{thmprfeqn}
&\mathcal{\bar{E}}_{g}(f)-\mathcal{\bar{E}}_{g}(f')= \nonumber \\
&\textstyle{\mathcal{E}^{\rho}_{e}(f)-\mathcal{E}^{\rho}_{e}(f')+
\frac{8M(2M-1)}{\rho}(\bar{R}_{m}(\mathcal{\hat{F}}_{L})-\bar{R}_{m}(\mathcal{\hat{F}}_{[L]\setminus l'}))}.
\end{align}
Since we have \(\bar{R}_{m}(\mathcal{\hat{F}}_{[L]\setminus l'})<\bar{R}_{m}(\mathcal{\hat{F}}_{L})\) in Lemma \ref{lemma6}, we have \(8M(2M-1)\rho^{-1}(\bar{R}_{m}(\mathcal{\hat{F}}_{L})-\bar{R}_{m}(\mathcal{\hat{F}}_{[L]\setminus l'}))>0\).
In addition, we have \(\mathcal{E}_{e}^{\rho}(f')-\mathcal{E}_{e}^{\rho}(f)=\epsilon\) where \(\epsilon\) is a positive value.
Then, when the condition of \(\textstyle{\mathcal{E}^{\rho}_{e}(f')-\mathcal{E}^{\rho}_{e}(f)<}\)
\(\textstyle{8M(2M-1)\rho^{-1}
(\bar{R}_{m}(\mathcal{\hat{F}}_{L})-\bar{R}_{m}(\mathcal{\hat{F}}_{[L]\setminus l'}))}\) holds,
we have \(\mathcal{\bar{E}}_{g}(f)-\mathcal{\bar{E}}_{g}(f')>0\) from (\ref{thmprfeqn}).
Therefore, since we have the bounds of \(\mathcal{\bar{E}}_{g}(f)\) and \(\mathcal{\bar{E}}_{g}(f')\) with probability \(1-\delta\) from Lemma \ref{lemma3}, we have the bound of \(\mathcal{\bar{E}}_{g}(f')<\mathcal{\bar{E}}_{g}(f)\) in Theorem \ref{theorem1} with probability \((1-\delta)^2\).
\end{proof}
Theorem \ref{theorem1} suggests that we can have tight upper bounds of generalization error by erasing layers after the training.
In terms of the retraining, we can retrain the models with good initial parameters that give tight upper bounds of generalization error.
Notice that although Lemma \ref{lemma2} suggests that the representational power may be reduced by erasing layers, Theorem \ref{theorem1} suggests that the model can obtain tight upper bounds of generalization error by erasing layers from trained model.
In other words, there exists the trade off between the representational power and upper bounds of generalization error when we erase layers after the training.
Therefore, if we erase layers form ResNet after training, we may obtain the model that has better generalization performance compared with the model before the erasure.

As just described, since the value of \(\textstyle{\mathcal{E}^{\rho}_{e}(f')-\mathcal{E}^{\rho}_{e}(f)}\) is small by erasing a layer,
the condition of \(\textstyle{\mathcal{E}^{\rho}_{e}(f')-\mathcal{E}^{\rho}_{e}(f)<\frac{8M(2M-1)}{\rho}(\bar{R}_{m}(\mathcal{\hat{F}}_{L})-\bar{R}_{m}(\mathcal{\hat{F}}_{[L]\setminus l'}))}\) is not a strict condition when we erase a layer.
The problem is the case wherein we erase multiple layers.
In this case, the condition may not hold because the value of \(\textstyle{\mathcal{E}^{\rho}_{e}(f')-\mathcal{E}^{\rho}_{e}(f)}\) can be large in Theorem \ref{theorem1}.
This may incur large drop in accuracy.
Our algorithm to eliminate the problem is introduced in the next section.

\subsection{Algorithm}
Theorem \ref{theorem1} suggests that we can effectively reduce the number of layers by retraining ResNet after removing Residual Units from trained ResNet in theory.
In order to realize this ``erasure and retraining'' scheme, we need to solve two additional problems:

{\bf (i) Identifying unimportant residual units.}
ResNet has several important layers whose erasure will dramatically degrade the accuracy.
In fact, since ResNet computes new representations when a stage is changed, several layers after the change are important in terms of achieving accuracy \cite{unroll}.
Since the erasure of important layers makes the value of \(\textstyle{\mathcal{E}^{\rho}_{e}(f')-\mathcal{E}^{\rho}_{e}(f)}\) large, it is difficult to hold Theorem \ref{theorem1}.
Therefore, we need to determine the importance of each Residual Unit in order to select those whose erasure will not drastically degrade accuracy.

To solve the problem, we introduce priority, measure of Residual Unit importance.
This priority can be learned from training data in the same way as other parameters in ResNet.
We can erase top-\(k\) Residual Units in increasing order of priority.
In particular, we determine the priority of \(F(\cdot)\) in (\ref{residual_unit}).
Notice that we have \({\bf x}_{l+1} = {\bf x}_{l}\) from (\ref{residual_unit}) by erasing \(F(\cdot)\) as shown in \cite{ensambles}.
This is equivalent to erasing the Residual Unit for (\ref{residual_unit}) because input \({\bf x}_{l}\) passes as output to \({\bf x}_{l+1}\) without change.
In order to determine the importance of \(F(\cdot)\), we use the following weighted Residual Unit:
\begin{align}
\label{weighted_residual_unit}
{\bf x}_{l+1} = {\bf x}_{l} + w_{l}F({\bf x}_{l}),
\end{align}
where \(w_{l}\) is a scalar that can be learned by back propagation in the same way as other ResNet parameters.
If \(w_{l}\) is small in terms of absolute value, it scales down the output of \(F(\cdot)\). 
In other words, \(F(\cdot)\) has little impact on the result if \(w_{l}\) is small in terms of its absolute value.
Therefore, we can select top-\(k\) unimportant nonlinear mappings \(F(\cdot)\) according to the values of \(|w_{l}|\).
However, we should not erase the first Residual Unit in a stage: it is the first Residual Unit after the dimensionality of \({\bf x}_{l}\) changed.
Although most Residual Units have outputs of the same dimensionality as inputs in (\ref{residual_unit}), the first Residual Unit in a stage changes the dimensionality of the inputs.
\cite{unroll} suggests that these Residual Units are important with regard to accuracy because they produce new representations in ResNet by changing the dimensionality.
Therefore, we use original Residual Units of (\ref{residual_unit}) as the first Residual Units in each stage and do not erase them.

\begin{algorithm}[t]
   \caption{Network Implosion.}
   \label{algorithm:ni}
\begin{algorithmic}[1]
   \REQUIRE training set \(D\), initial learning rate \(\eta\), number of Residual Units \(L\), number of Residual Units to be erased at a time \(k\), total number of Residual Units to be erased \(L'\), number of epochs for retraining \(n\)
   \STATE Initialize parameters of layers in ResNet and weights \(\{w_{l}:l \in[L]\}\).
   \STATE Train the model by using SGD with \(\eta\).
   \STATE Set \(l'=L\) and \(s=0\).
   \WHILE{\(l'>L'\)}
   \STATE \(s \gets s+1\)
   \STATE Set \(I_{s}=\emptyset\).
   \STATE Select top-\(k\) small elements in absolute values of \(\{w_{l}:l\in[L]\}\) and add the indices into \(I_{s}\).
   %%\STATE {\bf Option}: check the condition of \(\hat{W}_{i}>1\) for \(i\in I_{s}\).^M
   \STATE Erase \(w_{i}F({\bf x}_{i})\) in (\ref{weighted_residual_unit}) where \(i\in I_{s}\).
   %\STATE {\bf Option}: check the condition of \(\mathcal{E}_{e}^{\rho}(f')>\mathcal{E}_{e}^{\rho}(f)\).^M
   %%\STATE {\bf Option}: check the condition of \(\textstyle{\mathcal{E}^{\rho}_{e}(f')-\mathcal{E}^{\rho}_{e}(f)}\)^M
%%\(\textstyle{<\frac{8M(2M-1)}{\rho}(\bar{R}_{m}(\mathcal{\hat{F}}_{L})-\bar{R}_{m}(\mathcal{\hat{F}}_{[L]\setminus I_{s}}))}\).^M
   \STATE Retrain the model by using SGD with \(\eta\) for \(n\) epochs.
   \STATE \(l' \gets l'-k\)
   \ENDWHILE
\end{algorithmic}
\end{algorithm}
{\bf (ii) Recovering accuracy for erasing multiple layers.}
\cite{ensambles} reports that accuracy decreases significantly when multiple layers are erased.
Since this also incurs the problem that the value of \(\textstyle{\mathcal{E}^{\rho}_{e}(f')-\mathcal{E}^{\rho}_{e}(f)}\) in Theorem \ref{theorem1} drastically increases, it is difficult to retain accuracy even if we retrain the model.

For this problem, we {\bf gradually erase Residual Units and retrain the network after each erasure}.
Specifically, our algorithm has four steps as follows;
(i) train the network as usual;
(ii) erase a few Residual Units according to \(|w_{l}|\), e.g. \(k=1\);
(iii) retrain network until the number of retraining epochs reaches \(n\), the number of epochs for retraining;
(iv) repeat (ii) and (iii) until we erase a specified number of layers.
By erasing a few layers at one time, we tend to keep the value of \(\textstyle{\mathcal{E}^{\rho}_{e}(f')-\mathcal{E}^{\rho}_{e}(f)}\) small whereas the value is large when we erase many layers at one time.
As a result, we can perform retraining with preferable initial parameters in terms of the bound of generalization error as shown in Theorem \ref{theorem1}.
Note that \cite{deepcompression} retrains the network one time after erasing parameters in order to maintain the accuracy.
However, this approach fails if we erase multiple Residual Units.
This is because we erase, at one time, more parameters than the previous method.

In addition, we {\bf retrain the network after the erasure with large learning rate}.
When we erase multiple layers, the structure of the network drastically changes.
Thus we need to effectively change the parameters in the remaining layers to efficiently recover the accuracy.
To realize this, we set a large learning rate when we retrain the network by using training algorithms such as Stochastic Gradient Descent (SGD) \cite{robbins}\cite{ida}.
In particular, we reuse the original learning rate of initial network training.

Algorithm \ref{algorithm:ni} shows the procedure of Network Implosion.
It repeatedly trains and erases Residual Units as explained above.
First, we train ResNet with the learning rate \(\eta\) (line 2).
Next, we erase top-\(k\) nonlinear mappings \(F(\cdot)\) according to the importance of \(|w_{l}|\) (lines 5-8).
Then, we retrain ResNet with the initial learning rate \(\eta\) (line 9).
We repeat the procedure until the accuracy drops or we erase sufficient numbers of layers (lines 4-11).
%If the conditions in the options do not hold, we stop the algorithm.
%Though we must execute the options in the procedure to follow the theory, they may be skipped to reduce the training cost in actual use.
\section{Experiments}
\begin{figure*}[t!]
\begin{center}
\begin{tabular}{c}
  \begin{minipage}{0.32\hsize}
  \begin{center}
  \includegraphics[viewport = 0.000000 0.000000 504.000000 504.000000, scale=0.22]{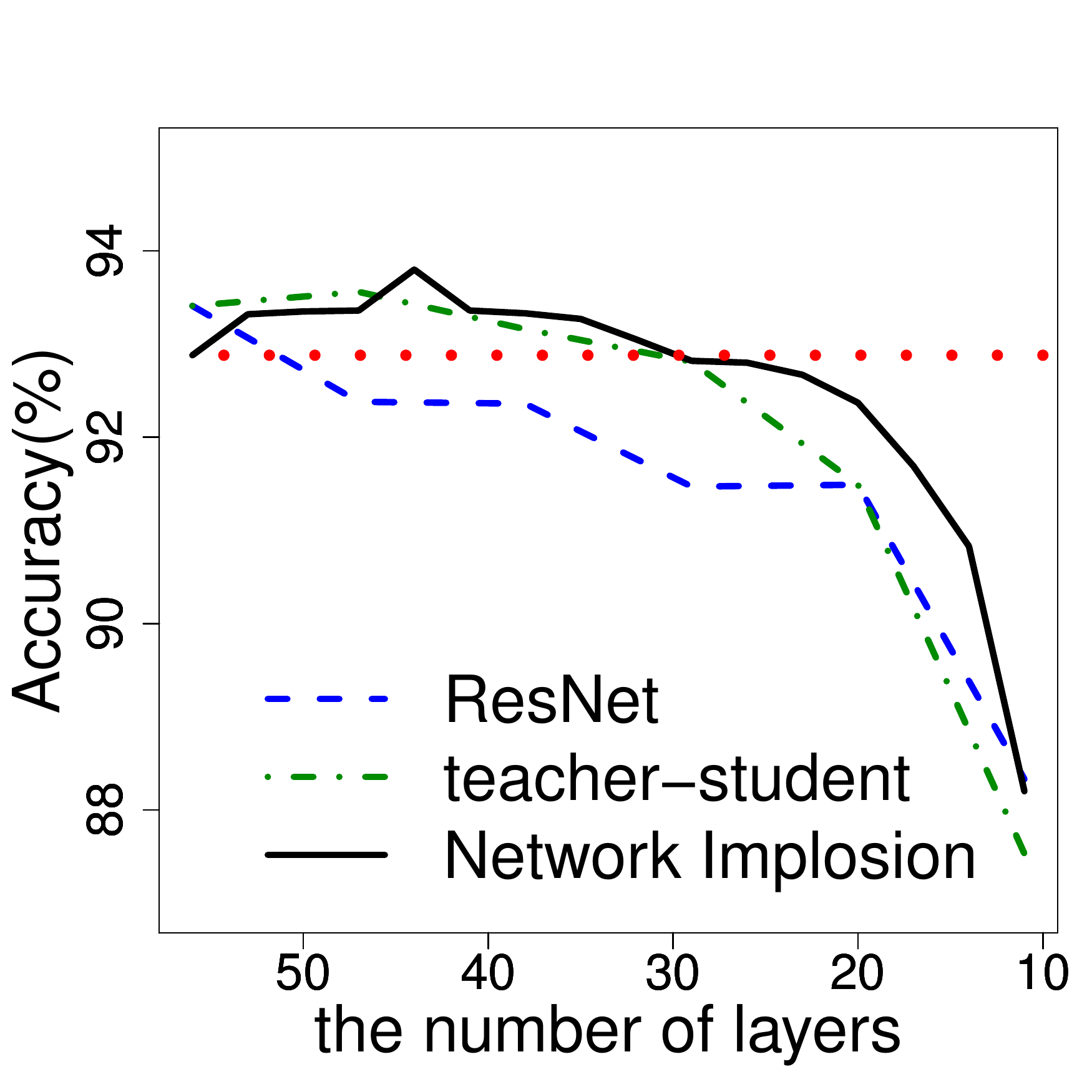} 
  \\{\small (a) Cifar-10}
  \end{center}
  \end{minipage}
  \begin{minipage}{0.32\hsize}
  \begin{center}
  \includegraphics[viewport = 0.000000 0.000000 504.000000 504.000000, scale=0.22]{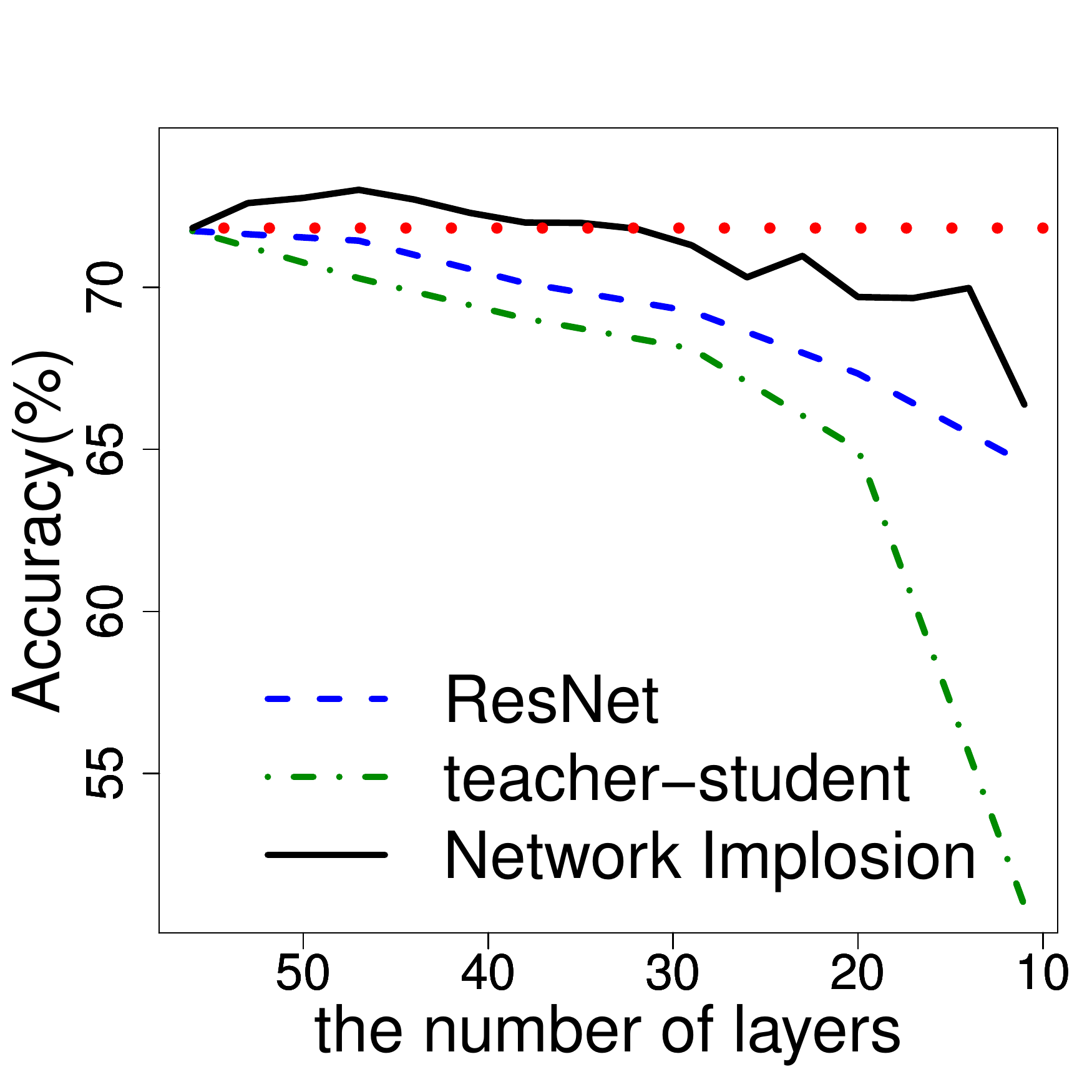} 
  \\{\small (b) Cifar-100}
  \end{center}
  \end{minipage}
  \begin{minipage}{0.32\hsize}
  \begin{center}
  \includegraphics[viewport = 0.000000 0.000000 504.000000 504.000000, scale=0.22]{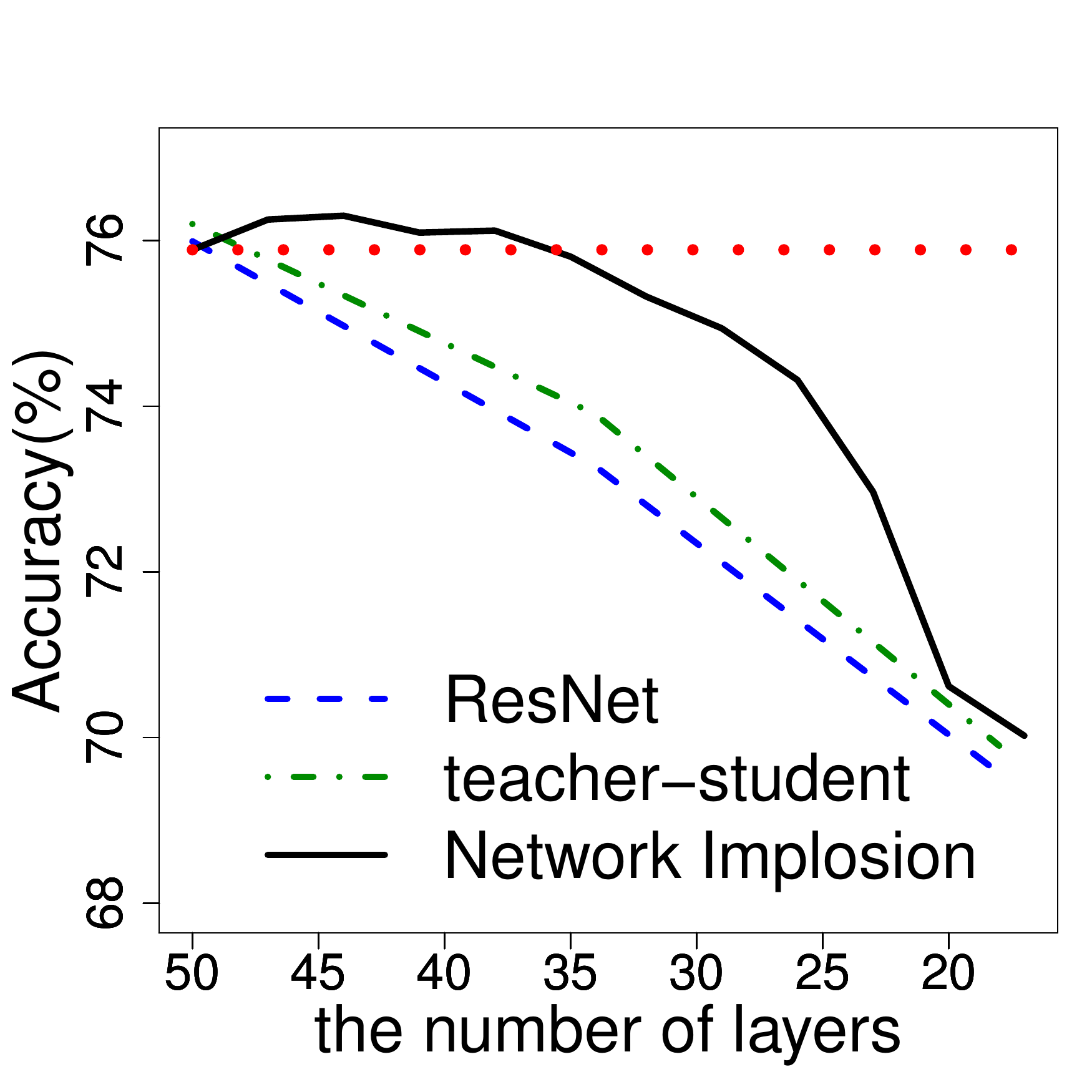} 
  \\{\small (c) ImageNet}
  \end{center}
  \end{minipage}
\end{tabular}
  \caption{Accuracies achieved for Cifar-10, Cifar-100, and ImageNet. The red dotted lines represent accuracies for initial models in our approach: Network Implosion. It achieves higher accuracies than baselines even though it erases many layers.}
  \label{fig_accuracy}
\end{center}
%\vskip -2em
\end{figure*}
We performed experiments to evaluate Network Implosion.
We first investigated the accuracy of our method, and then computational costs for the inference phase.
We implemented our approach in Torch7 \cite{Torch}.

We performed image classification tasks using Cifar-10, Cifar-100 \cite{cifar} and ImageNet dataset \cite{imagenet} from ILSVRC2012.
Cifar-10, Cifar-100 and ImageNet have 10, 100 and 1000 classes, respectively.
The image size is \(32\times32\times3\) for Cifar-10/100.
For ImageNet, we scaled the images \(256\times256\times3\) and took \(224\times224\times3\) single center-crop for training and testing by following \cite{googlenet}.
We applied color, scale and aspect ratio augmentation to the images the same as \cite{alexnet,googlenet}.

We implemented ResNet following \cite{1001layer} as it is used for many image classification tasks.
Specifically, each Residual Unit has three convolutional layers with batch normalization and ReLU  forming a bottleneck structure.
The numbers of stages are three for Cifar-10/100, and four for ImageNet.
As a result, the number of layers is 56 for Cifar-10/100, and 50 for ImageNet, the same as \cite{resnet}.
We used projection shortcuts \cite{1001layer} for increasing dimensions; this is used when stages are changed.
The other hyper parameters were also set according to \cite{1001layer} and fb.resnet.torch\footnote{https://github.com/facebook/fb.resnet.torch}
which is widely used in the deep learning community;
the number of epochs was 200;
the training algorithm was SGD with momentum;
the momentum was 0.9;
the initial learning rate was 0.1;
the learning rate was divided by 10 at 81 and 122 epochs for the Cifar-10 and Cifar-100 datasets.
For ImageNet, we decayed the leaning rate by multiplying the learning rate by \(0.1\) at every 30 epochs.
The mini-batch size was 128 for Cifar-10/100, and 512 for ImageNet.
The weight decay was 0.0001.
The parameters in each layer were initialized as in \cite{he}, a standard method for deep neural networks with ReLU activations.

For comparison, we also used the teacher-student training regime based on Knowledge Distillation \cite{KD}\cite{DDN}.
The teacher-student training regime can reduce the number of layers without additional computation costs for the inference while previous dynamic layer pruning cannot as described in the section of related work.
In addition, teacher-student training regime can directly detemines the number of layers while previous static layer pruning cannot.
We used 56 and 50 layered ResNets as the teacher networks for Cifar-10/100 and ImageNet, respectively.
The temperature was 4, and we used \(\alpha=0.9\) in Cifar-10/100 (see \cite{KD} for a description of these hyper parameters).
In ImageNet, the temperature was 1 and we used 0.5 for the tunable parameter.

In our approach, we used Residual Units of (\ref{weighted_residual_unit}) when \({\bf x}_{l+1}\) had the same dimensionality as \({\bf x}_{l}\).
When the dimensionality of \({\bf x}_{l+1}\) differed from \({\bf x}_{l}\), we used the original Residual Units of (\ref{residual_unit}) and did not erase them.
This is because such Residual Units are important for generating new representations as described in the previous sections.
We erased one Residual Unit (k=1) after each training or retraining cycle.
Since each Residual Unit has three convolution layers, we can erase three layers by erasing a Residual Unit.
In the retraining phase, we used the initial learning rate of 0.1.
The number of epochs was 60.
We divided the learning rate by 10 at 20 and 40 epochs.

\subsection{Accuracy}
We evaluated the validation accuracies of our approach, teacher-student training regime, and original ResNet.
For Cifar-10/100, we reduced the number of layers from 56 to 11.
For ImageNet, we trained 50, 34 and 18 layer models for original ResNet and teacher-student training regime while our method reduced the layers from 50 to 17. 

Fig. \ref{fig_accuracy} shows the experimental results.
The red dotted lines in the figure are initial accuracies of Network Implosion; these are accuracies of 56 and 50 layer models for Cifar10/100 and ImageNet, respectively.
Although Network Implosion erases layers from the model, it yielded accuracies above the red dotted lines by erasing layers.
This is because we can retrain the models with preferable initial parameters as described in Theorem \ref{theorem1}.
This result verifies our theoretical result.
Finally, we could reduce the number of layers to 32, 35 and 38 for Cifar-10, Cifar-100 and ImageNet without accuracy loss, respectively.
These numbers correspond to crossover points of the red and black lines in Fig. \ref{fig_accuracy}.
In other words, we could reduce the number of layers by 42.86\%, 37.50\% and 24.00\% for Cifar-10, Cifar-100 and ImageNet, respectively.
Notice that when we simply reduced the number of layers in original ResNet, the accuracies fell even though the number of training epochs was the same as our method (blue dashed lines in the figures).
Although the teacher-student training regime achieves comparable accuracies to our method for Cifar-10, it rapidly degrades the accuracies as more layers were eliminated for Cifar-100 and ImageNet (green dot-dash lines in the figures).
In particular, it drastically degraded the accuracy at 34 layers whereas our algorithm kept the accuracy high for ImageNet.
These results reveal that Network Implosion is effective in reducing the number of layers even if we use real-world datasets such as ImageNet.
\begin{table*}[t!]
  \caption{The computation costs for the inference phase after erasing layers without accuracy loss. 56 and 50 layered models are original models for Cifar-10/100 and ImageNet, respectively. Our method reduces all computation cost without accuracy loss.}
  \label{time}
  \begin{center}
  \small
  \begin{tabular}{|c|c|c|c|c|c|c|}
  \hline
%dataset & \# of layers & accuracy (\%) & \# of multiply-accumurate operations (MACs) & forward (msec) & backward (msec) & \# of parameters \\
dataset & \shortstack{\# of\\layers} & \shortstack{accuracy\\(\%)} & \shortstack{\# of\\MACs} & \shortstack{forward\\(msec)} & \shortstack{backward\\(msec)} & \shortstack{\# of\\parameters} \\
  \hline
Cifar-10 & 56 & 92.88 & \(8.19\times10^7\) & \(6.584\) & \(12.93\) & 585.9K \\ \cline{2-7}
         & {\bf 32} & {\bf 93.05} &  \({\bf 4.99\times10^7}\) & {\bf 3.970} & {\bf 7.721} & {\bf 409.1K} \\
  \hline
Cifar-100 & 56 & 71.83 & \(8.65\times10^7\) & \(6.203\) & \(13.36\) & 613.6K \\ \cline{2-7}
          & {\bf 35} & {\bf 71.99} & \({\bf 5.44\times10^7}\) & {\bf 4.350} & {\bf 8.075} & {\bf 555.3K} \\
  \hline
ImageNet & 50 & 75.89 & \(4.11\times10^9\) & \(29.95\) & \(59.51\) & 25.55M \\ \cline{2-7}
         & {\bf 38} & {\bf 76.12} & \({\bf 3.23\times10^9}\) & {\bf 22.97} & {\bf 46.53} & {\bf 23.80M} \\
  \hline
  \end{tabular}
  \end{center}
\end{table*}
\subsection{Computation Costs}
We evaluated the computation costs for the inference phase: the number of MAC (multiply-accumulate) operations, the processing times of forward and backward propagations, and the number of parameters.
MAC is the main operation of deep neural networks and is used in convolution and fully-connected layers.
As the processing times, we averaged 100 runs for forward and backward propagation.
In addition, we counted the number of parameters to be learned for evaluating model sizes.
We used the same setting as the previous section for training.
By using Network Implosion, we reduced 56-layered models to 32 and 35 layers for Cifar-10 and Cifar-100, respectively.
For ImageNet, we reduced the 50-layered model to 38 layers.
These models are the smallest models with no drop in accuracy as described in Fig. \ref{fig_accuracy}.

Table \ref{time} shows the results.
The table shows that the numbers of MACs are reduced to 60.93\%, 62.89\% and 78.59\% of baselines for Cifar-10, Cifar-100, and ImageNet, respectively.
In proportion to the number of MACs, we could reduce the processing times of forward propagation to 60.23\%, 70.13\% and 76.69\% of baselines for each dataset.
For backward propagation, which is used for fine-tuning in Transfer Learning \cite{trans}, we could achieve 59.71\%, 60.44\% and 78.19\% of the processing times for the respective datasets.
In terms of model size, our approach reduced the number of parameters to 69.82\%, 90.50\% and 93.15\% of baselines for Cifar-10, Cifar-100 and ImageNet with no drop in accuracy, respectively.
The results reveal that our approach erases layers without additional computation costs or drop in accuracy for the inference phase.

\section{Conclusion}
We proposed Network Implosion that can erase multiple layers from ResNets with no loss of accuracy.
It offers high accuracy by using priority to select which Residual Units to erase; the remaining units are retrained.
Our theoretical results guarantee that the ``erasure and retraining'' scheme can easily recover the accuracy:
we can retrain the model with preferable initial parameters in terms of generalization error bound.
We evaluated our approach on Cifar-10, Cifar-100 and ImageNet.
The results show that Network Implosion effectively reduces the number of layers without degrading the accuracy even on real-world datasets such as ImageNet.

\bibliography{neural_implosion.bib}
\bibliographystyle{unsrt}

\end{document}